%% file: customer-version.tex
\pdfoutput=1

\documentclass[11pt]{article}

\usepackage[final]{acl}

\usepackage{listings}
\usepackage{xcolor}
\definecolor{codegreen}{rgb}{0,0.6,0}
\definecolor{codegray}{rgb}{0.5,0.5,0.5}
\definecolor{codepurple}{rgb}{0.58,0,0.82}
\definecolor{backcolour}{rgb}{0.95,0.95,0.92}

\lstdefinestyle{mystyle}{
    backgroundcolor=\color{backcolour},   
    commentstyle=\color{codegreen},
    keywordstyle=\color{magenta},
    numberstyle=\tiny\color{codegray},
    stringstyle=\color{codepurple},
    basicstyle=\ttfamily\footnotesize,
    showspaces=false,                
    showstringspaces=false,
    showtabs=false,                  
    tabsize=2
}
\usepackage{times}
\usepackage{latexsym}
\usepackage{subcaption}
\usepackage{amsmath}
\usepackage{amsfonts}
\usepackage{amsthm}
\usepackage{amssymb}
\usepackage{booktabs}
\usepackage{diagbox}
\usepackage[T1]{fontenc}
\usepackage[utf8]{inputenc}
\usepackage{microtype}
\usepackage{inconsolata}
\usepackage{graphicx}
\usepackage{multicol}
\usepackage{multirow}
\usepackage{hyperref}
\usepackage[most]{tcolorbox} 
\definecolor{block-gray}{gray}{0.95}
\newtcolorbox{zitat}[2][]{%
    colback=block-gray,
    grow to right by=0mm,
    grow to left by=0mm, 
    boxrule=0pt,
    boxsep=0pt,
    breakable,
    enhanced jigsaw,
    borderline west={4pt}{0pt}{gray},
    colbacktitle={block-gray},
    coltitle={black},
    fonttitle={\large\bfseries},
    attach title to upper={},
    #1,
}
\newtheorem{theorem}{Theorem}
\usepackage[skip=0.5\baselineskip]{caption}

\usepackage{xcolor,pifont}
\newcommand*\colourcheck[1]{%
  \expandafter\newcommand\csname #1check\endcsname{\textcolor{#1}{\ding{52}}}%
}
\newcommand*\colourx[1]{%
  \expandafter\newcommand\csname #1x\endcsname{\textcolor{#1}{\ding{55}}}%
}
\colourcheck{blue}
\colourcheck{green}
\colourcheck{red}
\colourx{blue}
\colourx{green}
\colourx{red}
\newcommand\blfootnote[1]{%
  \begingroup
  \renewcommand\thefootnote{}\footnote{#1}%
  \addtocounter{footnote}{-1}%
  \endgroup
}

\usepackage{enumitem}
\setitemize{noitemsep,topsep=0pt,parsep=0pt,partopsep=0pt,leftmargin=*}

\title{Learning LLM Preference over Intra-Dialogue Pairs: A Framework for Utterance-level Understandings}

\author{Xuanqing Liu$^*$, Luyang Kong$^*$, Wei Niu, Afshin Khashei, Belinda Zeng, \\ {\bf Steve Johnson, Jon Jay, Davor Golac, Matt Pope}\\
Amazon.com Inc.
}

\begin{document}
\maketitle
\begin{abstract}
Large language models (LLMs) have demonstrated remarkable capabilities in handling complex dialogue tasks without requiring use case-specific fine-tuning. However, analyzing live dialogues in real-time necessitates low-latency processing systems, making it impractical to deploy models with billions of parameters due to latency constraints. As a result, practitioners often prefer smaller models with millions of parameters, trained on high-quality, human-annotated datasets. Yet, curating such datasets is both time-consuming and costly. Consequently, there is a growing need to combine the scalability of LLM-generated labels with the precision of human annotations, enabling fine-tuned smaller models to achieve both higher speed and accuracy comparable to larger models. In this paper, we introduce a simple yet effective framework to address this challenge. Our approach is specifically designed for per-utterance classification problems, which encompass tasks such as intent detection, dialogue state tracking, and more. To mitigate the impact of labeling errors from LLMs -- the primary source of inaccuracies in student models -- we propose a noise-reduced preference learning loss. Experimental results demonstrate that our method significantly improves accuracy across utterance-level dialogue tasks, including sentiment detection (over $2\%$), dialogue act classification (over $1.5\%$), etc.\blfootnote{$^*$First two authors contributed equally. Corresponding author email: \texttt{xuanqing@amazon.com}}
\end{abstract}

\section{Introduction}

Maintaining high annotation quality, scaling the size of labeled datasets, and managing annotation budgets are three critical yet often conflicting objectives in deploying real-world ML applications. A widely adopted paradigm involves a two-stage process: unsupervised pretraining followed by supervised fine-tuning (e.g., \citealp{devlin2018bert, chen2020simple, he2020momentum, raffel2020exploring}). This approach effectively reduces the size of the labeled dataset required because, during the pretraining phase, models learn to generate universal embeddings across various modalities. Consequently, such pretrained models are often straightforward to adapt to downstream tasks.
\par
In dialogue understanding, moving beyond BERT-like models is essential, as dialogues possess unique characteristics compared to the BERT pretraining corpus (which primarily consists of books and web pages). These differences arise from several factors: First, dialogues involve spoken language exchanges between two or more individuals and are often structured differently, with one line per speaker. This format reduces the effectiveness of tasks such as masked token prediction and next-sentence prediction. Second, the vocabulary in daily dialogues tends to be informal. Finally, dialogues are frequently transcribed from voice recordings, introducing ASR errors and background noise. These distinctive properties have inspired research into developing specialized unsupervised pretraining algorithms for dialogue data (\citealp{mehri2019pretraining, zhong2022dialoglm, liu2022dial2vec, zhou2022learning}). Benchmark evaluations on common dialogue tasks -- such as intent detection, next-utterance prediction, summarization, dialogue act classification, and dialogue state tracking -- demonstrate the advantages of dialogue-optimized models. These models generally adhere to the classical BERT framework, pretraining on large-scale unsupervised dialogue datasets with dialogue-specific loss functions, including random mask filling, utterance swapping, and contrastive learning. However, it remains unclear whether such pretrained embedding models generalize effectively to specific downstream tasks.
\par
To address this challenge, we require direct supervision signals that are closely aligned with downstream tasks. This motivates the use of instruction fine-tuned LLMs as phase-2 supervision signals, while retaining traditional unsupervised pretraining as phase-1. However, simply employing LLMs as data labelers and fine-tuning a student model using traditional cross-entropy loss proves suboptimal. The accuracy of LLM-generated labels can be unpredictable, influenced by factors such as the quality of the LLM, the prompting strategy, and the inherent difficulty of the dialogue task. Consequently, the knowledge transferred from the LLM to the student model often deviates from the intended objective. This paper proposes an alternative approach based on preference learning, where pairs of chunks sampled from the same dialogue session (\emph{intra-session pairs}) are labeled by ensembled LLMs. Under reasonable assumption on LLM labeling errors, our method outperforms traditional training algorithms in both data efficiency and generalizability.
\section{Related work}
\subsection{Task-oriented dialogue (TOD) system}
Task-oriented dialogue understanding lies in the core of building AI assistants to be deployed in domain specific scenarios such as restaurant booking, self-service product troubleshooting, and so on. The objective is to help users achieve their goals in limited turns by understanding users' needs, tracking dialogue states and figure out next best action. Unique to TOD system, intent detection, dialogue act classification, and dialogue state tracking are three critical components of the system. Traditional approaches mostly rely on supervised learning on embedding models~\cite{liu2016attention}, by encoding dialogue contexts and employing deep neural networks such as RNN/LSTM or Transformers to infer utterance labels or slot values~\cite{barriere-etal-2022-opinions,btoasis,chen2020simple}. In the LLM age, there is a shift from finetuning TOD model for a specific domain~\cite{lei2018sequicity} to open domain in-context learning~\cite{hu2022context,arora2024intent}. Unfortunately, both solutions ignored latency and cost constraints in real-time, commercial products.

\subsection{Synthetic label prompting strategies and transfer learning}
These two techniques are the foundation of our solution. We discuss the main idea and prior works.
\par
\noindent\textbf{Prompting strategies}. It is often non-trivial prompting LLMs to achieve quality high data labeling. For example, prior work~\cite{anagnostidis2024susceptible,workrethinking,lu2021fantastically} noticed that few-shot prompting is surprisingly sensitive to factors including the number of example, order of examples, positive / negative sample ratio, or how similar those examples are to the actual input query. In this regard, fine-tuning embedding models on human curated labels are still preferred in production-ready applications. To strengthen the robustness of ICL, a promising solution is through diversified prompting~\cite{li2023self,song2024preference,song2024scaling}, either by starting with a few seeding prompts, and augment more versions using automated pipeline~\cite{wang2022self}, or repetitively refine the prompt from diverse perspectives~\cite{li2023dail}.
\par
\noindent\textbf{Transfer learning}. For better instruction following ability, a popular approach is fine-tuning on synthetic datasets produced by larger LLMs~\cite{taori2023stanford,chiang2023vicuna,xu2023wizardlm}. To foster LLM's reasoning ability, another line of work finetune with synthetic rationales collected from stronger LLMs~\cite{wang2022pinto,shridhar2023distilling,liu2023logicot,kang2024knowledge}.
Similar approach work for task-specific applications too, examples like dialogue generation~\cite{xu2023baize}, information extraction~\cite{josifoski2023exploiting,jeronymo2023inpars} and code generation~\cite{chaudhary2023code,roziere2023code}.
Our work focus on per-utterance multi-class classification in TOD system, assuming that even the most capable LLMs can't generate highly accurate labels, so a brand new transfer learning approach is required.
\section{Proposed framework}
\subsection{Problem scope}
We limit our scope to per-utterance classification, including sentiment detection, dialogue state tracking, dialogue act classification (Fig.~\ref{fig:illustrate-context-dependent}).

\begin{figure}[htb]
    \centering
    \begin{subfigure}{0.47\textwidth}
    \caption{\label{fig:example-intent-detection}Intent detection}
    \includegraphics[width=\textwidth]{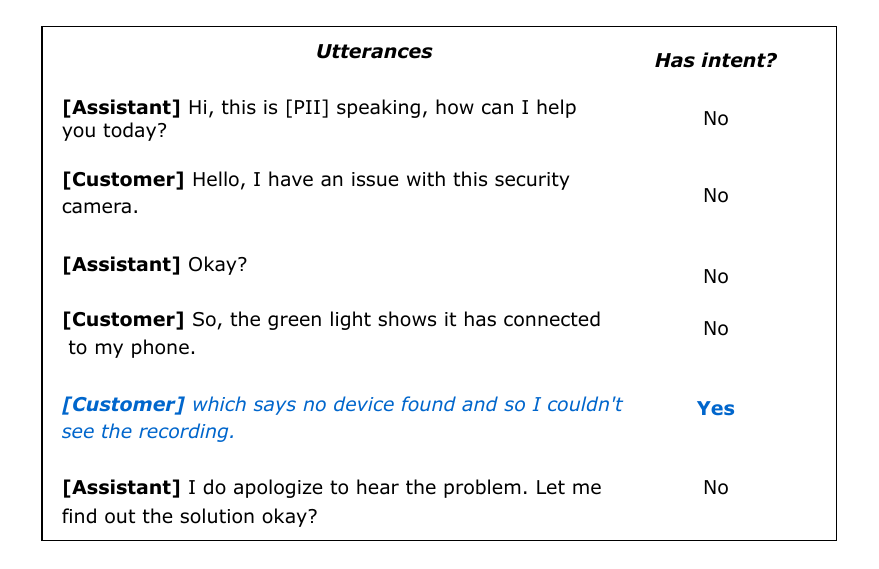}
    \end{subfigure}
    \par
    \begin{subfigure}{0.47\textwidth}
    \caption{\label{fig:example-da-classification}Dialogue act classification}
    \includegraphics[width=\textwidth]{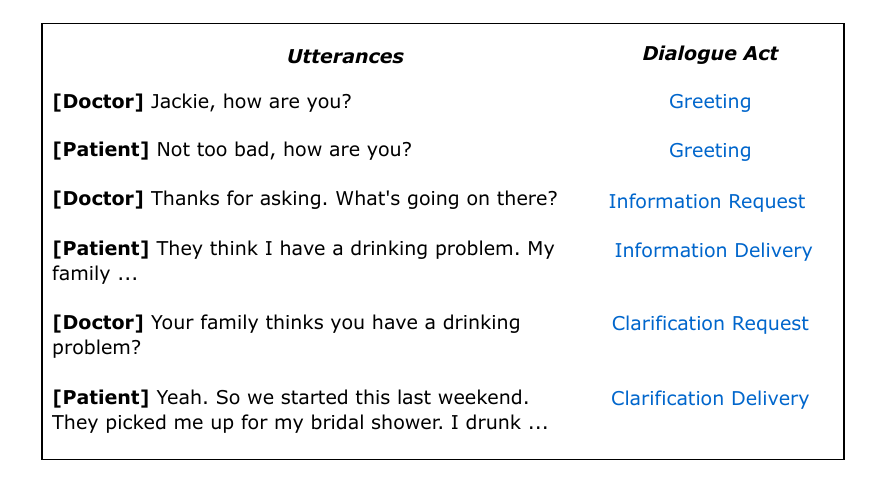}
    \end{subfigure}
    \par
    \begin{subfigure}{0.47\textwidth}
    \caption{\label{fig:example-dialogue-state-tracking}Dialogue state tracking}
    \includegraphics[width=\textwidth]{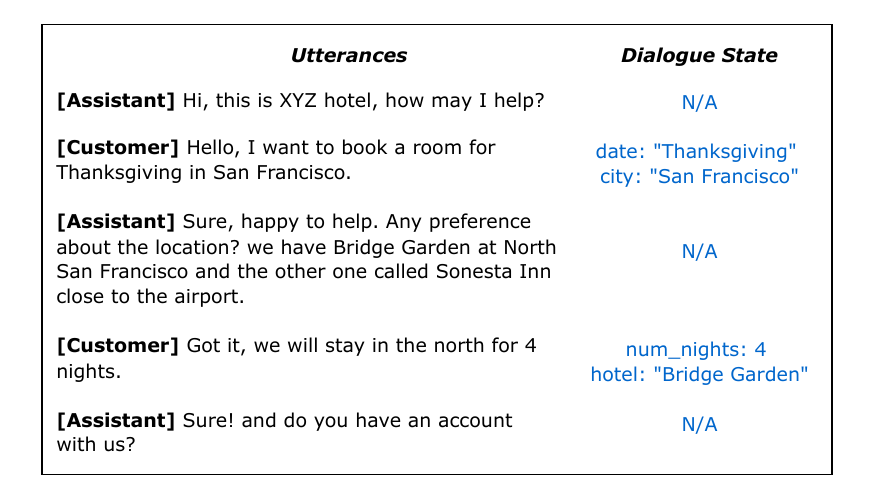}
    \end{subfigure}
    \caption{Illustrative examples of intent detection, dialogue act classification, and dialogue state tracking problems.}
    \label{fig:illustrate-context-dependent}
\end{figure}
\par
\noindent\textbf{Intent detection}. Each utterance is mapped to a binary label \texttt{has\_intent} ($y=1$) or \texttt{no\_intent} ($y=0$). Positive label means utterance deemed a valid intent (e.g. a question, issue, or complaint). Take customer support for example, we could apply intent detection model to monitor customer speech in real time and figure out whether a customer is seeking for help rather than chit-chatting.
\par
\noindent\textbf{Dialogue act classification}. We could regard this as an extension of intent detection from binary intent labels to multi-class acts. The objective of dialogue act classification is finding out the functions that utterances serve in dialogues -- such as commitments, questions, requests, replies, etc. In contact centers, for example, classifying dialogue acts can be valuable at providing appropriate and thoughtful responses to clients adhering to the dialogue acts.
\par
\noindent\textbf{Dialogue state tracking (DST)}. The objective of DST is extracting and picking up new information into dialogue state as the conversation evolves. This task has great potential in customer service as it not only provides intent types (e.g. \emph{hotel-booking} in Fig.~\ref{fig:example-dialogue-state-tracking}), but also identifies relevant semantic concepts throughout the slot filling process (e.g. \emph{location = San Francisco}).
\par
\noindent\textbf{Challenge.} When delivering real world applications driven by per-utterance classifiers, the challenges often rooted from obtaining high quality labels. For example, MultiWOZ~\cite{budzianowski2018multiwoz} is commonly used for benchmarking DST algorithms. Yet the original dataset contains numerous labeling errors, and it took $4$ future versions~\cite{eric2019multiwoz,zang2020multiwoz,han2021multiwoz,ye2021multiwoz} (MultiWOZ 2.1-2.4) to correct them. More importantly, we learned that a clean dataset not only ensures us precisely tracking the progress on good valid/test set, but also reduces the reliance on robust model training algorithms~\cite{ye2022assist}. The challenge of labeling leads us to focus on following question --
\begin{zitat}{}
\emph{Can we design a general solution for per-utterance classification problems, by jointly utilizing small amount of clean, human verified labels and almost unlimited amount of lower quality LLM annotations?}
\end{zitat}
We share a positive answer in the remainder of this work. Our work is not a simple extension of weakly supervised learning or noise-robust supervised learning, as we utilize characteristics that are unique to per-utterance classifications.


\subsection{Workflow}

\begin{figure*}[htb]
    \centering
    \includegraphics[width=0.95\textwidth]{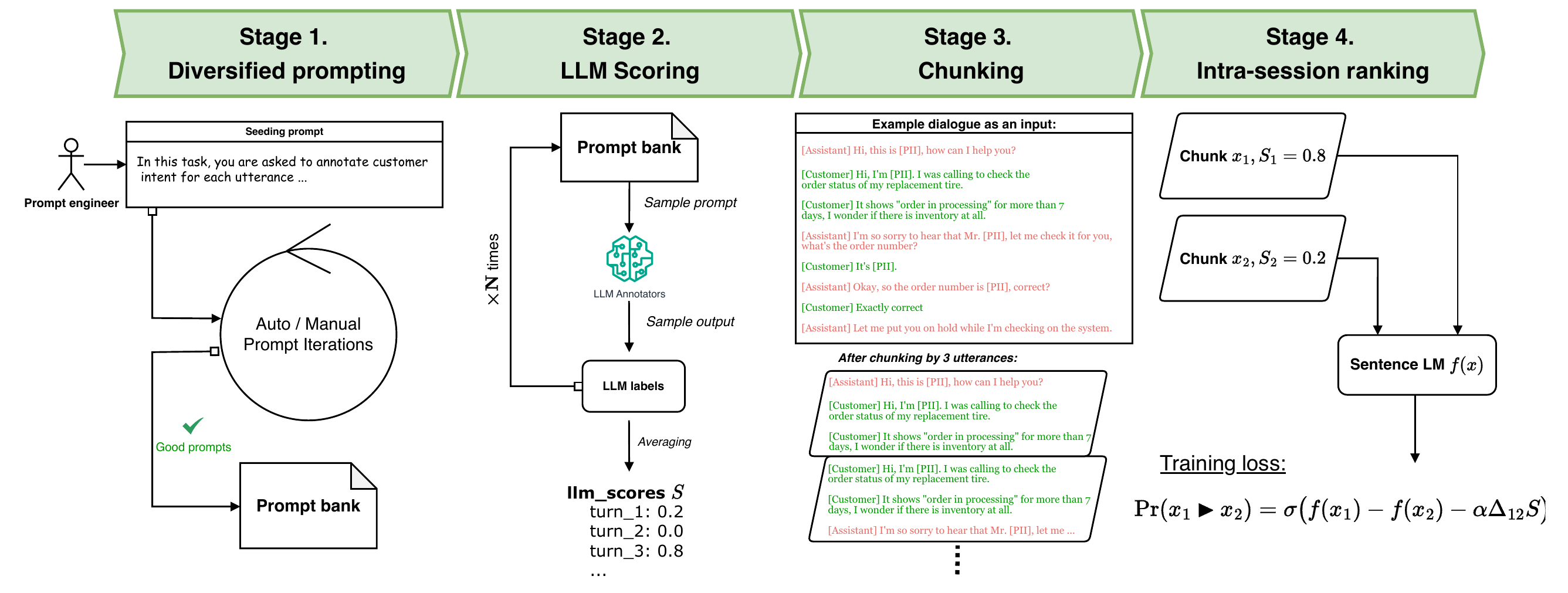}

    \caption{\label{fig:workflow-overview}Overview of our framework to train a small student model using noisy LLM supervision.}

    \label{fig:workflow}
\end{figure*}

Our workflow involves four stages. Goal of stage 1 is to construct a \emph{prompt bank} containing diversified prompts that performs well on data annotation work following prompt tuning strategies outlined in \citealt{schulhoff2024prompt,brown2020language,wei2022chain,yao2023treethoughtsdeliberateproblem,liu2021makes}. Predictions led by various prompts are slightly different, we ensemble the outputs together for better results~\cite{khalifa2023exploring,jiang-etal-2021-know}. Next, we further strengthen the ensemble effect at stage 2 using top-$K$/top-$P$ sampling. After repeated sampling $N$ times using LLM labeler, we compute $L$-dimensional score vector $S\in[0,1]^{L}$ for dialogue $\mathcal{D}$ containing $L$ utterances. Each element $0\le S_{i}\le 1$ is the ratio of positive LLM labels divided by $N$ (e.g. if $3$ in $10$ ensembles labeled $i$-th utterance as positive, $S_i=0.3$). For $C$-class classification problem, we transform it into $C$ one-versus-rest binary classification problems so the same framework still apply.
\par
After we collect LLM labeling scores $S$, we split a dialogue into multiple segments using a sliding window of stride $1$. 
We denote $x_i$ as the $i$-th segment covering $u_1$ to $u_i$. Finally in stage 4, we randomly sample two \emph{intra-session} segments $x_i$ and $x_j$ from the same dialogue and train a student model $f$ minimizing pair-wise ranking loss:

\begin{equation}\label{eq:loss-function}
    \ell(x_i, x_j)=\text{KL}\big(\mathbb{I}_{y_i\blacktriangleright y_j}  \parallel \mathrm{Pr}(x_i\blacktriangleright x_j)\big),
\end{equation}
where $\mathbb{I}_{y_i\blacktriangleright y_j}=1$ iff. $y_i=1$ and $y_j=0$ for binary labels; $\mathrm{Pr}(x_i\blacktriangleright x_j)$ is the probability of $x_i$ being more positive than $x_j$, modeled by network $f$ under an adaptive margin:

\begin{equation}\label{eq:margin-probability}
\mathrm{Pr}(x_i\blacktriangleright x_j) = \sigma\big(\Delta_{i,j} f - \alpha\cdot\Delta_{i,j} S\big),
\end{equation}
where $\sigma$ is the Sigmoid function, $\Delta_{i,j} f=f(x_i) - f(x_j)$ is the difference of model predicted scores and $\Delta_{i,j}S = S_i - S_j$ is the difference of LLM predicted scores between segment $i$ and $j$; $\alpha\in [0, 1]$ is a tunable hyper-parameter controlling margin. We train a student network $f$ over intra-session pairs to ensure: for any positive+negative pair labeled by LLM (positive $x_i$ \emph{vs.} negative $x_j$), the student network $f$ has the same preference as teacher LLM under margin $\alpha\cdot\Delta_{i,j}S$. This idea made two hidden assumptions: First assuming the LLM score $S$ is a good estimator of ground-truth correctness probability (\textit{aka.} confidence calibrated~\cite{pmlr-v70-guo17a}); secondly, single LLM labeler may be biased and high variance, their difference within same dialogue session $S_i-S_j$ carries dramatically lower bias and variance due to the differentiation. Therefore estimation error of $S_i-S_j$ is more precise than $S_i$ or $S_j$ alone. We discuss and verify two assumptions in the following sections. 

\subsection{Stage 1-2: How well are LLM scores calibrated to accuracy?}

A desirable property of LLM teacher is confidence scores $S$ calibrated to labeling accuracy, i.e. we expect higher true-positive rate if LLM score $S_i$ closes to one; and near zero true-positive rate if $S_i$ is closer to zero:

\begin{equation}\label{eq:def-calibrated}
\mathrm{Pr}(y_i=1 | S_i) = S_i.
\end{equation}

If Eq.~\eqref{eq:def-calibrated} is true, we could replace ground truth label $y_i$ with soft label $S_i$ without incurring additional gradient bias and variance (see Appendix~\ref{sec:app-proof-unbias-thm} for a proof). In addition, Eq.~\eqref{eq:def-calibrated} implies monotonicity relationship: 

\begin{equation}\label{eq:monotonic_s}
S_i>S_j \Longrightarrow \mathrm{Pr}(y_i=1) > \mathrm{Pr}(y_j=1).
\end{equation}

\cite{pmlr-v70-guo17a} showed that DNNs are uncalibrated, in that their accuracy falls behind confidence score (DNNs are over-confident). Same findings are reported in LLM world~\cite{kapoor2024calibration,huang2024calibrating}. Among various post-training solutions to calibrate DNNs (e.g. \cite{zadrozny2001obtaining,mozafari2018attended}), 
one simple and effective technique is ensemble different models~
\cite{lakshminarayanan2017simple} which integrates well with our workflow. Remaining question to be answered in this work is - 
\begin{zitat}{}
\emph{Does the same ensemble technique work for LLM predictions? If so, how many ensemble predictions we need to calibrate the scores?}
\end{zitat}
We design following experiment to answer this question: We sample an intent detection dataset containing around $600$ transcripts and binary \texttt{has\_intent} / \texttt{no\_intent} per-utterance labels. A labeling prompt optimized for \texttt{Claude3-sonnet}\footnote{Available at \href{https://docs.anthropic.com/en/api/messages}{Anthropic} and \href{https://docs.aws.amazon.com/bedrock/}{AWS Bedrock}.} for this task is provided in Appendix~\ref{sec:app-prompts-intent}. We apply the same prompt to ensemble sizes $n$ between $1$ and $30$. In each setting, we run LLM labeling on each input pair $\langle x_i, x_j\rangle$ for $n$ times and obtain scores $S_i$ and $S_j$ by averaging LLM predictions. Lastly, we partition the data by value $S_i$ into five buckets: $S_i\in(0.0, 0.2]$, $(0.2, 0.4]$, $(0.4, 0.6]$, $(0.6, 0.8]$, $(0.8, 1.0]$. Within each bucket, we compute the percentage of positive ground-truth labels.
We apply ECE loss, the standard metric to measure DNN calibration error~\cite{pmlr-v70-guo17a}:

\begin{equation}\label{eq:ECE-loss}
 \text{ECE} = \sum_{m=1}^{M}\frac{|B_m|}{N}\Big|\text{acc}(B_m) - \text{conf}(B_m)\Big|
\end{equation}

where $B_m$ is the $m$-th bucket partitioned by $S_i$. $\text{acc}(B_m)=\text{Pr}(y_i=1|s_i\in B_m)$ is the accuracy of $B_m$; and $\text{conf}(B_m)$ is the overall confidence score in $B_m$. Due to Eq.~\eqref{eq:def-calibrated} lower ECE metric means better calibration.
\begin{figure}[htb]
    \centering
    \includegraphics[width=0.9\linewidth]{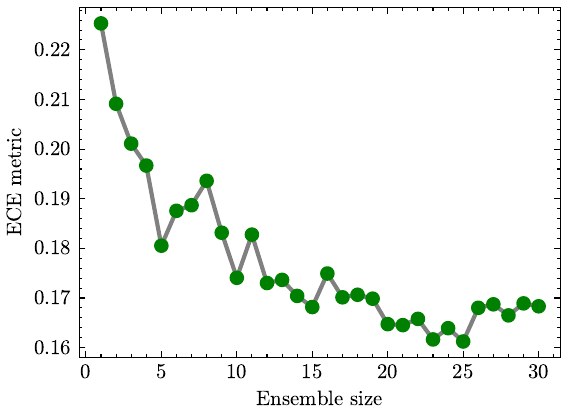}
    \caption{\label{fig:ece-loss-curve}Visualizing the downward trend of ECE loss as ensemble size increases from 1 to 30.}
\end{figure}
Despite some random fluctuations, we could observe in Fig.~\ref{fig:ece-loss-curve} a decline in ECE loss ($0.22\searrow 0.17$) as ensemble size increases.
\begin{zitat}{}
\emph{The ensemble technique in Stage 1-2 effectively calibrates LLM scores $S_i$ by introducing fewer gradient biases and variances. Therefore LLM teacher supervisions are good surrogate for ground-truth labels.}
\end{zitat}
\subsection{Stage 3-4: Overcoming distribution shifts by intra-session comparison}
We generate ranking pairs in a novel way: we sample two chunks for ranking from the same conversation (\emph{intra-session pairs}), instead of different conversations. We make two hypothesis ($H_1$ and $H_2$) explaining why intra-session pairs are more powerful.
\par
\noindent\textbf{{\color{blue}$H_1$}: Intra-session pairs are harder}. Two chunks sampled from same dialogue are similar in the context (sharing the same topic with overlapping context). As a result, it is harder to tell which chunk is positive label against the other. Once training a student model on top of hard pairs, it forces the model to learn more discriminative textual features from text input, rather than just replying on some keywords. Those intra-session pairs lead to better generalization.
\par
\noindent\textbf{{\color{blue}$H_2$}: LLM labeling errors are canceled by the differentiator}. This hypothesis is more conceptually involved: LLM labeling errors are not uniformly random across all data, instead they cluster on certain type of transcripts. For example, some scenarios are not mentioned in the labeling prompt so LLM has to guess, resulting in more errors in such cases. Fortunately, this type of error typically condensed to certain dialogues, equivalent to a ``shifting'' effect to the label distribution. By sampling a pair ($x_i$ and $x_j$) from the same dialogue, their corresponding LLM scores ($S_i$ and $S_j$) are drifted to roughly the same extent. In the end, the margin of the loss function~\eqref{eq:loss-function} $\Delta_{ij} S = S_i - S_j$ still accurately tracking ground-truth label difference $y_i - y_j$.
\par
\begin{figure}[htb]
    \centering
    \includegraphics[width=0.9\linewidth]{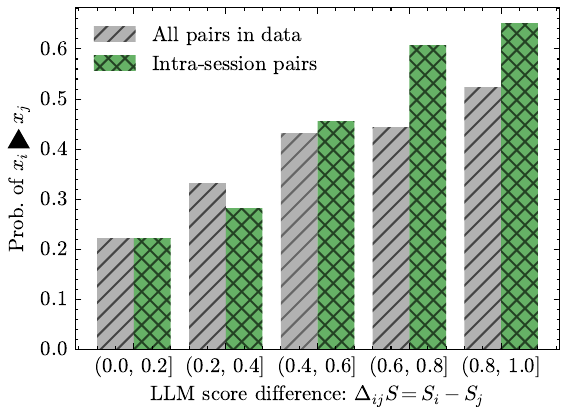}
    \caption{Comparing the correlations between LLM score difference (also the margin of training loss) \emph{w.r.t.} the probability of one label is more positive than the other. We also include linear fittings to both groups.}
    \label{fig:stylized_fact}
\end{figure}
We design an experiment to validate $H_2$ on two groups: the control group consists of pairs sampled from different dialogues; experimental group consists of pairs sampled from same dialogue. The goal is checking correlation between $\Delta_{ij} S=S_i-S_j$ with the probability of $y_i=1$ and $y_j=0$ ($y_i>y_j$ in binary case). We follow the same bucketizing method as previous experiment (5 buckets). We count the percent of $y_i>y_j$ cases in each bucket and each group. Result in Fig.~\ref{fig:stylized_fact} shows the ground-truth probability of $y_i>y_j$ more sensitive to $\Delta_{ij} S$ in experimental group than control group. Meaning that our intra-session pairs are indeed less noisy, and a better approximation of golden supervision signal $y_i-y_j$. 

\section{Experiments}
\noindent\textbf{Datasets.} We benchmark our method on three important tasks in task-oriented dialogues (TOD): intent/sentiment-detection, dialogue act classification, and dialogue state tracking. We benchmark intent/sentiment detection on MELD~\cite{poria-etal-2019-meld} and SILICONE~\cite{busso10interactive}; benchmark dialogue act classification on daily-dialog~\cite{li2017dailydialog}, MRDA~\cite{shriberg2004icsi}, BT-OASIS~\cite{btoasis} and dyda\_da~\cite{chapuis-etal-2020-hierarchical}; benchmark dialogue state tracking on SGD~\cite{rastogi2020towards} and MultiWOZ-2.2~\cite{zang2020multiwoz}. We put statistics and other details of datasets in Appendix~\ref{app:summary-dataset}.
\par
\noindent\textbf{Baselines.}
We want to see how the accuracy change after plugging our workflow into some strong models. We select following baselines accordingly:
\begin{itemize}
\item \emph{Claude3-Sonnet}: We pick this model as a strong baseline for measuring LLM annotator performance.
\item \emph{FnCTOD}~\cite{li2024large}: A recent prompting strategy achieving strong results on dialogue state tracking task.
\item \emph{ToD-BERT}~\cite{wu2020tod}: A strong baseline for dialogue pretrained small embedding model. This is also the backbone model of our method.
\item \emph{FLAN-T5}~\cite{chung2024scaling}: T5-XXL fine-tuned on large-scale instructions data including MultiWOZ. We include this model as a natural baseline for fine-tuned LLM on TOD datasets.
\end{itemize}
We summarize features of all baselines with our method in Table~\ref{tab:characters} of Appendix~\ref{app:compare-model-features}.

\subsection{Comparing pairwise preference learning \emph{vs.} pointwise knowledge transfer}

\begin{table}[htb]
    \centering
    \scalebox{0.8}{
    \begin{tabular}{l|ccccc}
    \toprule
    \diagbox[width=10em]{Approach}{$\%$ gold labels} &  $0\%$  & $1\%$  & $5\%$  & $10\%$  & $25\%$ \\
    \midrule
    Finetune-only        & - & $27.3$ & $29.5$ & $34.7$ &  $69.6$ \\
    \midrule
    \multicolumn{6}{c}{\textit{Supervised pretrain $\rightarrow$ Finetune}}\\
    Pointwise pretrain & - & $31.8$ & $33.4$ & $47.2$ &  $77.3$ \\   
    Pairwise pretrain  & - & $38.4$ & $45.8$ & $52.1$ &  $78.4$ \\
    \bottomrule
    \end{tabular}}
    \caption{Effective of our approach under various amount of labeled data.}
    \label{tab:downsample-data-exper}
\end{table}
To evaluate the transition from pointwise model distillation to pairwise preference learning, we compare the intent detection accuracy of the ToD-BERT model fine-tuned using three approaches: 1) fine-tuning directly on human-labeled data; 2) supervised pretraining with pointwise LLM-generated labels followed by fine-tuning on human-labeled data; and 3) supervised pretraining with pairwise LLM-generated labels followed by fine-tuning on human-labeled data. To assess the impact of data scaling, we vary the sampling ratios during evaluation. Table~\ref{tab:downsample-data-exper} consistently shows that models leveraging pairwise supervised pretraining outperform the alternatives, particularly in low-data regimes.

\subsection{Sentiment detection}
Next we benchmark our method with baselines on two sentiment detection datasets. We report classification accuracy over all sentiments defined in each datasets. The results are shown in Table~\ref{tab:bench_intent_sentiment}. Comparing with ToD-BERT (finetuned directly on human labeled data) and FnCTOD (finetuned on LLM synthetic data), our approach (supervised pretrained on LLM synthetic data using pairwise loss then finetuned on human labeled data) performs better than baselines by around $2\%$ to $8\%$.
\begin{table}[htb]
    \centering
    \scalebox{0.7}{
    \begin{tabular}{l|ccccc}
    \toprule
    Datasets &  Claude & FnCTOD & ToD-BERT & FLAN-T5  & \textbf{Ours}     \\
    \midrule
    MELD     & 74.25 &  68.84 & 80.30 &  75.72 & 88.09  \\
    IEMOCAP  & 76.39 &  61.30 & 87.88 &  82.62 & 90.31  \\
    \bottomrule
    \end{tabular}
    }
    \caption{Benchmarking intent/sentiment detection task.}
    \label{tab:bench_intent_sentiment}
\end{table}
\subsection{Dialogue act classification}
Similarly, we benchmark our method against baselines on dialogue act classification problem. Note we adopted the same backbone model as ToD-BERT, and ToD-BERT is still the strongest baseline in this task. Our model out-performed ToD-BERT by around $1.5\%$ to $10\%$.
\begin{table}[htb]
    \centering
    \scalebox{0.7}{
    \begin{tabular}{l|ccccc}
    \toprule
    Datasets &  Claude & FnCTOD & ToD-BERT & FLAN-T5  & \textbf{Ours}    \\
    \midrule
    DailyDialog & 70.39 & 66.03 & 72.40 & 68.08 & 76.50    \\
    MRDA        & 62.82 & 81.93 & 88.4 & 60.47 & 89.95    \\
    dyda\_da    & 71.25 & 74.82 & 79.14 & 68.66 & 85.11    \\
    BT-Oasis    & 32.85 & 52.76 & 59.24 & 17.13 & 69.62    \\
    \bottomrule
    \end{tabular}
    }
    \caption{Benchmarking dialogue act classification task.}
    \label{tab:bench_dialogue_act}
\end{table}
\subsection{Dialogue state tracking}
Finally, we benchmark on two dialogue state tracking (DST) datasets, SGD and MultiWOZ-2.1. In this experiment we benchmark the accuracy of joint prediction of slot/domain/values (aka. \textbf{Joint-Acc}). The results are shown in Figure~\ref{tab:bench_dialogue_state}.
\begin{table}[htb]
    \centering
    \scalebox{0.7}{
    \begin{tabular}{l|ccccc}
    \toprule
    Datasets      &  Claude & FnCTOD & ToD-BERT & FLAN-T5  & \textbf{Ours}    \\
    \midrule
    SGD    & 60.7 & 63.9 & 42.5 & -- & 47.3 \\
    MultiWOZ  & 27.0 & 37.9 & 16.4 & --  & 25.5  \\
    \bottomrule
    \end{tabular}
    }
    \caption{Benchmarking dialogue state tracking task.}
    \label{tab:bench_dialogue_state}
\end{table}
\section{Discussion and future work}

This paper presents a novel approach to minimizing human effort in labeling high-quality data for a class of per-utterance classification problems. Our method moves beyond traditional LLM labeling and knowledge transfer to student models by leveraging a preference learning and pairwise ranking framework. This framework has been demonstrated to be both theoretically and empirically robust against LLM labeling errors. An intriguing future direction would be to extend this approach to reward model training in reinforcement learning with human feedback (RLHF), another critical domain characterized by noisy labels and the need for robust discriminative model training.

\clearpage
\bibliography{custom}
\appendix
\include{appendix.tex}
\end{document}

%% file: appendix.tex
\section{\label{app:summary-dataset}Summary statistics of experiment datasets}
\begin{table}[htb]
    \centering
    \resizebox{\columnwidth}{!}{
    \begin{tabular}{l|lll}
    \toprule
    Data  & \#Classes & \#Dialogues & \#Utterances \\
    \midrule
    \multicolumn{4}{l}{\it Intent/Sentiment detection} \\
    MELD        &  $3$  &   $1,400$   &  $13,000$ \\
    IEMOCAP     &  $6$  &   $151$   &   $10,039$ \\
    \midrule
    \multicolumn{4}{l}{\it Dialogue act classification} \\
    DailyDialog &  $5$  &   $13,118$   &  $103,630$         \\
    MRDA        &  $5$  &   $75$   &  $108,202$         \\
    dyda\_da    &  $4$  &   $87170$   &  $102,000$         \\
    BT-Oasis    & $42$ &  $636$ & $15,067$    \\
    \midrule
    \multicolumn{4}{l}{\it Dialogue state tracking}                \\
    SGD          & $53$ (slots)  &   $16,142$  &  $329,964$        \\
    MultiWOZ-2.1 & $24$ (slots)  &   $8,438$   &  $42,190$         \\
    \bottomrule
    \end{tabular}
    }
    \caption{Datasets for each evaluation task and some statistics.}
    \label{tab:dataset_info}
\end{table}

\section{\label{app:compare-model-features}Comparing features of baseline models and our method}
\begin{table}[htb]
    \centering
    \scalebox{0.8}{
    \begin{tabular}{lccc}
    \toprule
    Methods     &  TOD finetuned?  &  LLM distilled  &  Small size  \\
    \midrule
    Claude & (unknown) & \redx & \redx \\
    FnCTOD & \redx & \greencheck & \redx \\
    ToD-BERT & \greencheck & \redx &   \greencheck \\
    FLAN-T5  & \greencheck &  \redx &   \redx \\
    Ours &  \greencheck & \greencheck & \greencheck    \\
    \bottomrule
    \end{tabular}
    }
    \caption{Comparing baselines and our method along three dimension: TOD finetuned means whether the model is finetuned for TOD tasks; LLM distilled indicates the model is distilled from (imperfect) LLM synthetic labels; Small size means whether the actual inference model is small footprint.}
    \label{tab:characters}
\end{table}

\section{Sample prompts for Claude}
\label{sec:app-prompts-trigger}

Prompt for daily-dialogue:
\begin{lstlisting}[language=HTML]
Dialogue: 
{dialogue}

Last utterance:
{last_utterance}

What's the best dialogue act of the last 
utterance? 
Choose from below without further 
explain:

Options:
A. Inform
B. Question
C. Directive
D. Commissive
E. None of above

A valid output should be one of: A, B, C, 
D, or E
Do not output anything else.
\end{lstlisting}

Prompt for MRDA:
\begin{lstlisting}[language=HTML]
Dialogue: 
{dialogue}

Last utterance:
{last_utterance}

What's the best dialogue act of the last 
utterance? Choose from below without 
further explain:

Options:
A. Statement or subjective statement
B. Declarative question
C. Backchannel
D. Follow-me
E. Question

A valid output should be one of: A, B, C, 
D, or E
Do not output anything else.
\end{lstlisting}

Prompt for MELD:
\begin{lstlisting}[language=HTML]
## Task Description

In this task you will receive a short 
dialogue. Your goal is to read the whole 
dialogue, understand the sentiment of 
each utterances, and pick out the utter-
ances with positive sentiment.

## Output format

You need to copy each positive sentiment 
utterances to an json array together 
with the initial line number.

## Example

Input:

1 [Phoebe] Oh my God, he's lost it. He's 
totally lost it. 
2 [Monica] What?
3 [Ross] Or! Or, we could go to the bank, 
close our accounts and cut them off at 
the source.
4 [Chandler] You're a genius!
5 [Joey] Aww, man, now we won't be bank 
buddies!
6 [Chandler] Now, there's two reasons.
7 [Phoebe] Hey.
8 [All] Hey!
9 [Phoebe] Ohh, you guys, remember that 
cute client I told you about? I bit him.
10 [Rachel] Where?!
11 [Phoebe] On the touchy.

Correct output:
```json
{
    "positive_utterances": [
        "4 [Chandler] You're a genius!",
        "8 [All] Hey!"
    ]
}
```
\end{lstlisting}

\section{Sample prompts for FLAN-T5}
\label{sec:flan-app-prompts-trigger}

Prompt for daily-dialogue:
\begin{lstlisting}[language=HTML]
Dialogue: 
{dialogue}

Last utterance:
{last_utterance}

What's the best dialogue act of the last
 utterance?

Options:
A. Inform
B. Question
C. Directive
D. Commissive
E. None of above
\end{lstlisting}

Prompt for MRDA:
\begin{lstlisting}[language=HTML]
Dialogue: 
{dialogue}

Last utterance:
{last_utterance}

What's the best dialogue act of the last
utterance? Choose from below without
further explain:

Options:
A. Statement or subjective statement
B. Declarative question
C. Backchannel
D. Follow-me
E. Question

Answer:
\end{lstlisting}

Prompt for MELD:
\begin{lstlisting}[language=HTML]
Dialogue: 
{dialogue}

Last utterance:
{last_utterance}

Is the last utterance in positive 
sentiment? Choose "Yes" or "No".
\end{lstlisting}

\section{Intent detection labeling prompt\label{sec:app-prompts-intent}}
\begin{lstlisting}[language=HTML]
# Task description
You are given a conversation between user
and assistant. Typically, the user has 
some questions / issues / complaints. 
Your goal is to find out the utterance
containing the user intent.

# Data description
Each line of the conversation corresponds
to an utterance. You can see the speaker
from according to the beginning of each
line. For example:

```
[assistant] Hi, my name is [PII], thank
you for calling [COMPANY].
[user] Hi, I'm calling because the
shippment arrived damaged and I need a
replacement.        
[assistant] I see, I'm sorry to hear
your bad experience about shippment. 
```

Here the user intent is "Hi, I'm calling
because the shippment arrived damaged
and I need a replacement.".

Now it is your turn, read the
conversation thoroughly and find out all
intent utterances

Conversation:
{conversation}
\end{lstlisting}

\section{Proof of Unbiased Gradients}\label{sec:app-proof-unbias-thm}

\begin{theorem}
Suppose dataset $\{(x_i, y_i)\}$ has binary labels $y_i\in\{0, 1\}$. If we only have access to noise-corrupted soft labels $\{x_i, \hat{y}_i\}$, $\hat{y}_i\in[0, 1]$ where the noisy labels follow the property $\mathrm{Pr}(y_i=1|\hat{y}_i)=\hat{y}_i$ (perfect confidence calibration). Then if we train a linear classifier $f_{\theta}(x) = \sigma(\theta^Tx)$ on corrupted dataset the gradients of cross-entropy loss over parameters $\theta$ are unbiased.
\end{theorem}
\begin{proof}
Training on corrupted dataset $\{x_i, \hat{y}_i\}$ using cross-entropy loss with linear model, we have the loss function:
\begin{equation}
    \label{eq:ce-loss}
    \begin{aligned}
    &L\big(\theta; (x_i, \hat{y}_i)\big)\\
    &=-\hat{y}_i\log\big(f_{\theta}(x_i)\big) - (1-\hat{y}_i)\log\big(1-f_{\theta}(x_i)\big)
    \end{aligned}
\end{equation}
If we compute the gradients of loss over parameters $\theta$:
\begin{equation}\label{eq:gradients}
\frac{\partial}{\partial \theta} L\big(\theta; (x_i, \hat{y}_i)\big)=\big(f_{\theta}(x_i)-\hat{y}_i\big)x_i.
\end{equation}
If we take the expectation over randomness of $\hat{y}_i$ on both sides of Eq.~\eqref{eq:gradients}, we can further get
\begin{equation}\label{eq:gradient-expansion}
\begin{aligned}
&\mathbb{E}\left[\frac{\partial}{\partial \theta}L(\theta; (x_i, \hat{y}_i))\right] \\
&=\big(f_{\theta}(x_i) - \mathbb{E}[\hat{y}_i]\big)x_i.
\end{aligned}
\end{equation}
Furthermore, due to the calibration of $\hat{y}_i$, $\mathrm{Pr}(y_i=1|\hat{y}_i)=\hat{y}_i$, we have that
\begin{equation}\label{eq:chain-exp}
\hat{y}_i=\mathrm{Pr}(y_i=1|\hat{y}_i)=\mathbb{E}[y_i|\hat{y}_i].
\end{equation}
Taking expectation on both sides in Eq.~\eqref{eq:chain-exp}, and leveraging the low of total expectation, we get
\begin{equation}\label{eq:total-exp}
\mathbb{E}[\hat{y}_i]=\mathbb{E}[\mathbb{E}[y_i|\hat{y}_i]]=\mathbb{E}[y_i].
\end{equation}
Finally, we plug Eq.~\eqref{eq:total-exp}
into Eq.~\eqref{eq:gradient-expansion}:
\begin{equation}
    \begin{aligned}
&\mathbb{E}\left[\frac{\partial}{\partial \theta}L(\theta; (x_i, \hat{y}_i))\right] \\
&=\big(f_{\theta}(x_i) - \mathbb{E}[\hat{y}_i]\big)x_i \\
&=\big(f_{\theta}(x_i) - \mathbb{E}[y_i]\big)x_i \\
&\mathbb{E}\left[\frac{\partial}{\partial \theta}L(\theta; (x_i, y_i))\right].
    \end{aligned}
\end{equation}
Therefore we have proved that well-calibrated training dataset $\{x_i, \hat{y}_i\}$ is unbiased training of the model.
\end{proof}